\documentclass[review]{elsarticle}

\usepackage{hyperref}
\usepackage{graphicx}
\usepackage{picinpar}
\usepackage{pdfpages}
\usepackage{amsmath}
\usepackage{url}
\usepackage{flushend}
\usepackage[latin1]{inputenc}
\usepackage{colortbl}
\usepackage{soul}
\usepackage{amssymb}
\usepackage{multirow}
\usepackage{pifont}
\usepackage{color}
\usepackage{alltt}
\usepackage{algorithm}
\usepackage{algorithmic}
\usepackage{enumerate}
\usepackage{siunitx}
\usepackage{breakurl}
\usepackage{epstopdf}
\usepackage{pbox}
\newtheorem{theorem}{Theorem}
\newtheorem{proof}{Proof}

\newtheorem{assumption}{Assumption}

\begin{document}

\begin{frontmatter}

\title{Multi-agent Actor-Critic with Time Dynamical Opponent Model}

% %% Group authors per affiliation:
% \author{Yuan Tian, Minghao Han,Chetan Kulkarni and Olga Fink\fnref{cor}}
% \address{Radarweg 29, Amsterdam}
% \fntext[cor]{Since 1880.}

\author[eth]{Yuan Tian}
\author[eth]{Klaus-Rudolf Kladny}
\author[eth]{Qin Wang}
\author[smu]{Zhiwu Huang}
\author[epfl]{Olga Fink \fnref{cor}}
\fntext[cor]{This work was supported by
the Swiss National Science Foundation under Grant PP00P2$\_$176878. Corresponding author: Olga Fink. Email: olga.fink@epfl.ch}

\address[eth]{ETH Z\"urich, Switzerland}
\address[epfl]{EPFL, Switzerland}
\address[smu]{Singapore Management University, Singapore}

% \address[mymainaddress]{1600 John F Kennedy Boulevard, Philadelphia}
% \address[mysecondaryaddress]{360 Park Avenue South, New York}

\begin{abstract}
In multi-agent reinforcement learning, multiple agents learn simultaneously while interacting with a common environment and each other. Since the agents adapt their policies during learning, not only the behavior of a single agent becomes non-stationary, but also the environment as perceived by the agent. This renders it particularly challenging to perform policy improvement. In this paper, we propose to exploit the fact that the agents seek to improve their expected cumulative reward and introduce a novel \textit{Time Dynamical Opponent Model} (TDOM) to encode the knowledge that the opponent policies tend to improve over time. We motivate TDOM theoretically by deriving a lower bound of the log objective of an individual agent and further propose \textit{Multi-Agent Actor-Critic with Time Dynamical Opponent Model} (TDOM-AC). We evaluate the proposed TDOM-AC on a differential game and the Multi-agent Particle Environment. We show empirically that TDOM achieves superior opponent behavior prediction during test time. The proposed TDOM-AC methodology outperforms state-of-the-art Actor-Critic methods on the performed experiments in cooperative and \textbf{especially} in mixed cooperative-competitive environments. TDOM-AC results in a more stable training and a faster convergence.
\end{abstract}


\end{frontmatter}

\section{Introduction}
Multi-agent systems have recently found applications in many different domains, including traffic control \cite{du2021learning}, games \cite{vinyals2019grandmaster,brown2019superhuman,OpenAI_dota}, consensus tracking control \cite{yin2022point,yuan2022suboptimal} and swarm control \cite{huttenrauch2019deep}. The complexity of the tasks in these applications often precludes the usage of predefined agent behaviors and stipulates the agents to learn a policy, and to define the problem as multi-agent reinforcement learning (MARL). In such cases, multiple agents learn simultaneously while interacting with a common environment. Since the agents adapt their policies during learning, not only the behavior of a single agent becomes non-stationary, but also the environment as perceived by the agents \cite{hernandez2017survey}. Since most of the conventional Reinforcement Learning (RL) approaches assume stationary system dynamics \cite{sutton1992reinforcement}, they usually perform poorly when required to interact with multiple adaptive agents in a shared environment \cite{lowe2017multi,hernandez2017survey}. 

A common approach in MARL is to explicitly consider the presence of opponents by modeling their policies using an opponent model \cite{brown1951iterative,tian2019regularized} (In the following, the word "opponents" refers to other agents in an environment irrespective of the environment's cooperative or adversarial nature). An accurate opponent model can provide informative cues to future behaviors of the opponents. However, such a precise prediction is challenging as the opponents' policies are changing over time \cite{tian2019regularized}.

 In our novel approach, entitled \textit{Time Dynamical Opponent Model} (TDOM), we aim to address the challenge of non-stationarity of the agent's behavior by modeling the opponent policy parameters as a dynamical system which are generally used to model the evolution of systems in time \cite{strogatz2018nonlinear}. Here, we build the system dynamics on the prior knowledge that all agents are concurrently trying to improve their policies with respect to their individual cumulative reward. It is worth mentioning that TDOM is highly general and can further support all kinds of opponent objectives, i.e. cooperative, competitive or mixed settings. 

% \begin{figure}[h]
% \centering
% \includegraphics[width=1\columnwidth]{IJCAI-ECAI-22/Fig/structure.png}
% \caption{Diagram of Multi-agent Actor-Critic with Time Dynamical Opponent Model learning algorithms. Under the CTDE framework, policies are updated by conventional policy improvement, while the opponent model is updated according to the proposed temporal improvement assumption. In the execution stage, only the modules in the orange rectangle are enabled. Each agent only has access to the state information and can only select an action based on its own prediction of other opponents' actions.}
% \label{Structure}
% \end{figure}

By deriving a lower bound on the log-objective of an individual agent, we further propose a Multi-agent Actor-Critic with Time Dynamical Opponent Model (TDOM-AC) for mixed cooperative-competitive tasks. The proposed TDOM-AC framework comprises a \textit{Centralized Training and Decentralized Execution} (CTDE). In this framework, centralized critics provide additional information to guide the training \cite{foerster2016learning,lowe2017multi}. However, this information is not used at execution time. Each agent only has access to the state information and can only select an action based on its own prediction of other opponents' actions.

We evaluate the proposed TDOM-AC on a Differential Game and a Multi-agent Particle Environment and compare the performance to two state-of-the-art actor-critic algorithms, namely \textit{Regularized Opponent Model with Maximum Entropy Objective} (ROMMEO) \cite{tian2019regularized} and \textit{Probabilistic Recursive Reasoning} (PR2) \cite{wen2019probabilistic}. We demonstrate empirically that the proposed TDOM algorithm achieves superior opponent behavior prediction during execution time. The proposed TDOM-AC outperforms the considered baselines on the performed experiments and considered measures. TDOM-AC results in a more stable training, faster convergence and \textbf{especially a superior performance in mixed cooperative-competitive environments}.

The remainder of this paper is organized as follows: Section \ref{sec:related work} provides a brief overview of the related works of this study. Section \ref{sec:method} and \ref{ac} introduces the proposed opponent model and TDOM-AC. Section \ref{sec:exp} interprets and compares the results of the experiments. In Section \ref{sec:conclusion}, the conclusion and future work are presented.

\section{Related Work}
\label{sec:related work}
Multi-Agent systems (MAS) encompass decision-making of multiple agents interacting in a shared environment \cite{kamdar2018state}. For complex tasks where using predefined agent behaviors is not possible, MARL enables the agent to learn from the interaction with the environment \cite{zhang2021multi}. One of the main challenges in MARL is the inherent non-stationarity. To address this challenge, one direction has been to account for the behaviors of other agents through a centralized critic by adopting the CTDE framework  \cite{foerster2018counterfactual,yang2018mean}. For value-based approaches in the CTDE framework, methods usually rely on restrictive structural constraints or network architectures, such as QDPP \cite{yang2020multi}, QMIX \cite{rashid2018qmix}, FOP \cite{zhang2021fop}, QTRAN \cite{son2019qtran}, and VDN \cite{sunehag2017value}. For actor-critic based methods, these approaches usually include an additional policy with supplementary opponent models that can reason about other agents' believes \cite{wen2019probabilistic}, private information \cite{tian2020learning}, behavior \cite{lowe2017multi}, strategy \cite{zheng2018deep} and other characteristics. With the supplementary opponent models, these works can also be linked to the field of opponent modeling (OM) \cite{albrecht2018autonomous,brown1951iterative}.

There are several ways to model the behavior of opponents. One of them is to factorize the joint policy $\pi(a^{-i}, \mathbf{a}^{-i}|s)$ in different ways. This has been done in previous works \cite{brown1951iterative,tian2019regularized,wen2019probabilistic}. Also, different objective functions for the opponent model have been implemented. \textit{Multi-agent Deep Deterministic Policy Gradient} (MADDPG) \cite{lowe2017multi} approximates opponents' policy by maximizing the log probability of other agents' actions with an entropy regularizer; PR2 \cite{wen2019probabilistic} considers an optimization-based approximation to infer the unobservable opponent policy via variational inference \cite{jordan1999introduction} and ROMMEO adopts the regularized opponent model with maximum entropy objective, which can be interpreted as a combination of MADDPG and PR2. However, the existing approaches either suffer from high computational cost due to the recursive reasoning policy gradient \cite{wen2019probabilistic}, or are limited to specific types of environments \cite{tian2019regularized}. In this work, we propose an alternative opponent model motivated by a temporal improvement assumption to overcome these limitations.

An earlier approach that explicitly addresses opponent-learning awareness is \textit{Learning with Opponent Learning Awareness} (LOLA) \cite{foerster2017learning}. When performing the policy update, any agent optimises its return under a one-step-look-ahead of the opponent learning. However, it is limited by strong assumptions. Specifically, these subsume access to both exact gradients and Hessians of the value function. Furthermore, a specific network design is required. Although the authors have subsequently proposed a variant of their approach, the \textit{policy gradient-based naive learner} (NL-PG) with fewer assumptions, the intrinsic on-policy design inherently suffers from data inefficiency. Also, LOLA only supports two-agent systems, while we are considering approaches that allow for arbitrarily many agents.

\section{Method}
\label{sec:method}
\subsection{Assumptions}
In this work, we aim to tackle the mentioned limitations outlined in Section \ref{sec:related work}. For fair comparison, we adopt the same observability assumptions from previous work \cite{wen2019probabilistic,tian2019regularized,lowe2017multi}. Since in cooperative games all the agents receive the same reward and in zero-sum games the opponents' rewards can easily be inferred from ones own reward, we assume all agents can access each other's rewards, just like in LOLA \cite{foerster2017learning}. In contrast to vanilla LOLA, we do not make the assumption of the observability of opponent policies.

\subsection{Markov Game}
An N-agents Markov game \cite{littman1994markov}, also referred to as N-agents stochastic game \cite{shapley1953stochastic}, is defined by a tuple $(s_t,a^1_t,...,a^n_t,r^1_t,...,r^n_t,p,\mathcal{T},\gamma)$. Within the tuple, $s_t$ is the state at time step $t$, $a^i_t$ and $r^i_t=r^i_t(s_t,a^i_t,\textbf{a}^{-i}_t)$ denote the set of actions selected by the policy of agent $i$ and the corresponding rewards assigned to agent $i$, where the $\textbf{a}^{-i}_t$ refers to the set of opponent actions. $\mathcal{T}$ is the state transition function, $p$ is the initial state distribution and $\gamma$ is the discount factor. At each time step $t$, actions are taken simultaneously by all agents. Each agent aims to maximize its own expected discounted sum of rewards.  Thus, for each individual agent $i$, the objective for its policy $\pi_i$ can be expressed as:
\begin{equation}
\begin{aligned}
    \centering
    J(\pi_i)=\text{max}_{\pi_i}\sum_{t=0}^\infty \mathbb{E}[\gamma^{t}r^i(s_t,a^i_t,\textbf{a}_t^{-i})]
\end{aligned}
\label{eq1}
\end{equation}

We note that since multiple adaptive agents interact in a shared environment, each agent's rewards and the environment transitions depend also on the actions of the opponents \cite{hernandez2017survey}. Thus, the unobservable  dynamic policies of the opponents induce non-stationarity in the environment dynamics from the perspective of a single agent. To address this challenge, we propose to consider the agent policy parameters as a dynamical system in which we encode the prior knowledge that all agents are concurrently trying to improve their policies.

\subsection{Time Dynamical Opponent Model} 
\label{TDOM}
% In this work, we embed the knowledge that all the opponents are simultaneously improving their policy according to their own objectives over time. 

% we assume that the observations of all agents are equal and refer to this common observation at time step $t$ as $s_t$ \cite{lowe2017multi}. Hence, the unknown variables to agent $i$ at time $t$ are the actions $\mathbf{a}_t^{-i}$ taken by all opponents $-i$ as a response to $s_t$.

To introduce our methodology, we begin by deriving a lower bound for the maximization objective in Equation \ref{eq1}, in which we omit some of the parameterization notation for less cluttering:

\begin{equation}
\begin{aligned}
\text{max}_{\pi^i, \; \rho^i}\mathbb{E}_{a^i_t \sim \pi^i(\cdot|\hat{\mathbf{a}}^{-i}_t), \;
\hat{\mathbf{a}}^{-i}_t \sim \rho^i, \; \mathbf{a}_t^{-i}\sim \tilde{\pi}^{-i}} \left[ Q^i(s_t, a_t^i, \mathbf{a}_t^{-i}) \right],
\end{aligned}
\end{equation}

\noindent where for lighter notation we omit the $s_t \sim d^{\pi}$, which means sampling a state from the discounted state visitation distribution $d^{\pi}$ using current policies $\pi := \{\pi^j\}_j$, where $\pi^j(\cdot|\mathbf{a}^{-j}, s)$. $\rho^i(\cdot|s)$ refers to the belief of agent $i$ about opponents $-i$, also known as \textit{opponent model}. Furthermore, we define $\tilde{\pi}^{j}(\cdot|s)$ to be

\begin{equation}
\tilde{\pi}^{j}(\mathbf{a}^{j}|s) := \int_{\mathcal{A}^{-j}} \; \pi^{j}(a^{j}|\mathbf{a}^{-j}, s) \; \rho^{j}(\mathbf{a}^{-j}| s) \; d\mathbf{a}^{-j},
\end{equation}

\noindent which can be interpreted as the marginal policy of agent $j$. Then we can formulate the marginal opponent policies to be $\tilde{\pi}^{-i} := \{\tilde{\pi}^j\}_{j \in -i}$

The presented maximization objective means that agent $i$ aims to maximize its Q-function given that all agents play their current policies $\tilde{\pi}_t^{-i}$ which are unknown to agent $i$.

We can now derive a lower bound of the log objective of agent $i$:

\begingroup
\addtolength{\jot}{1em}
\begin{equation}
\begin{aligned}
& \quad \text{log} \; \mathbb{E}_{a^i_t \sim \pi_t^i(\cdot|\hat{a}^{-i}_t), \; \hat{\mathbf{a}}^{-i}_t \sim \rho_t^i, \; \mathbf{a}_t^{-i}\sim \tilde{\pi}_t^{-i}} \left[ Q^i(s_t, a_t^i, \; \mathbf{a}_t^{-i}) \right] \\
=& \quad \text{log} \; \int_{\mathcal{A}^{i}} \int_{\mathcal{A}^{-i}} \int_{\mathcal{A}^{-i}} Q^i(s_t, a_t^i, \mathbf{a}_t^{-i}) \; \tilde{\pi}_t^{-i}(\mathbf{a}_t^{-i}| s_t) \; \rho_t^i(\hat{\mathbf{a}}^{-i}_t|s_t) \; \pi_t^i(a_t^i|\hat{\mathbf{a}}_t^{-i}, s_t) \; \\& \quad d\hat{\mathbf{a}}_t^{-i} \; d\mathbf{a}_t^{-i} \; d a^i_t \\
=& \quad \text{log} \; \int_{\mathcal{A}^{i}} \int_{\mathcal{A}^{-i}} \int_{\mathcal{A}^{-i}} Q^i(s_t, a_t^i, \mathbf{a}_t^{-i}) \; \frac{\tilde{\pi}_t^{-i}(\mathbf{a}_t^{-i}| s_t)}{\rho_t^i(\mathbf{a}^{-i}_t|s_t)} \; \rho_t^i(\mathbf{a}^{-i}_t|s_t) \; \rho_t^i(\hat{\mathbf{a}}^{-i}_t|s_t) \;\\&\quad \pi_t^i(a_t^i|\hat{\mathbf{a}}_t^{-i}, s_t) \; d\hat{\mathbf{a}}_t^{-i} \; d\mathbf{a}_t^{-i} \; da^i_t \\
\geq& \quad \; \mathbb{E}_{a^i_t \sim \pi_t^i(\cdot|\hat{\mathbf{a}}^{-i}_t), \; \hat{\mathbf{a}}^{-i}_t \sim \rho_t^i, \; \mathbf{a}_t^{-i}\sim \rho_t^{i}}\Bigg[ \text{log} \; Q^i(s_t, a_t^i, \mathbf{a}_t^{-i}) \; + \; \text{log} \; \left( \frac{\tilde{\pi}^{-i}(\mathbf{a}_t^{-i}| s_t)}{\rho_t^i(\mathbf{a}^{-i}_t|s_t)} \right) \Bigg] \\
=& \quad \; \mathbb{E}_{a^i_t \sim \tilde{\pi}_t^i, \; \mathbf{a}_t^{-i}\sim \rho_t^{i}}\left[ \text{log} \; Q^i(s_t, a_t^i, \mathbf{a}_t^{-i})\right] \; - \; \text{KL}\left(\rho_t^{i}(\cdot|s_t) \; || \; \tilde{\pi}_t^{-i}(\cdot|s_t)\right).
\end{aligned}
\end{equation}
\endgroup

If we furthermore make the assumption that:
\begin{equation}
Q^{\text{opt}} = \text{max}_{a^i} \; Q^i(s_t, a^i, \mathbf{a}^{-i}) \quad \forall \mathbf{a}^{-i} \in \mathcal{A}^{-i},
\end{equation}

\noindent for some fixed $Q^{\text{opt}}$, we see that we can maximize this lower bound by minimizing the Kullback-Leibler Divergence $\text{KL}\left(\rho_t^{i}(\cdot|s) \; || \; \tilde{\pi}_t^{-i}(\cdot|s_t)\right)$ w.r.t. $\rho_t^i$ and then maximizing the Q-function w.r.t. $\pi_t^i$:

\begin{equation}
\text{max}_{\pi_t^i} \; \mathbb{E}_{a^i_t \sim \pi_t^i(\cdot|\mathbf{a}^{-i}_t), \; \mathbf{a}_t^{-i}\sim \rho^i_{t}}\left[ \; Q^i(s_t, a_t^i, \mathbf{a}_t^{-i})\right]. 
\end{equation}

However, the method proposed above has an obvious issue: How can we minimize $\text{KL}\left(\rho_t^{i}(\cdot|s) \; || \; \tilde{\pi}_t^{-i}(\cdot|s_t)\right)$ if $\tilde{\pi}_t^{-i}$ is not available to agent $i$? 

In order to address this question, we propose to utilize prior information about the opponents' learning process. Using this information would enable to better model their non-stationary behavior. Specifically, there exists one aspect that to the best of our knowledge has not been considered before in opponent modelling: Over time, each agent $j$ is expected to improve its policy using policy network parameters $\theta^j_t$ \footnote{When parameterizing a function for agent $j$, we will always write e.g. $\pi^j_{\theta}$ instead of $\pi^j_{\theta^j}$.} in order to maximize its expected cumulative reward under the given system dynamics and opponent policies. This can be expressed as an ordinary differential equation (ODE):

\begingroup
\addtolength{\jot}{1em}
\begin{equation}
\begin{aligned}
\frac{d}{d t} \; \theta^j_t \; &\approx \; \nabla_{\theta^j} \; \mathbb{E}_{\pi^{j}_{\theta_t}, \; \pi_t^{-j}} \left[\sum_{t = 0}^{\infty} \gamma^t \; r^j_t(a^j_t, \; \mathbf{a}^{-j}_t, \; s_t) \right] \\
&= \; \nabla_{\theta^j} \; \mathbb{E}_{\pi_{\theta}^j, \; \pi_t^{-j}} \left[Q^{j}(a^j_t, \; \mathbf{a}^{-j}_t, \; s_t) \right],
\end{aligned}
\end{equation}
\endgroup
\noindent where $\pi^{-j} := \{ \pi^k \}_{k \in -j}$.

We propose to encode this knowledge in the opponent model design. We would like to make explicit here that unlike in policy improvement, the opponent model is designed to simulate the policy optimization process for all opponents instead of maximizing their expected Q-value. It is worth to point out that unlike LOLA \cite{foerster2017learning} which considers the opponent's policy update to optimize the agent's policy, our agent takes the opponents' policy improvement assumption into account to optimize its opponent model instead of the policy directly. 

In order to minimize $\text{KL}\left(\rho^{i}(\cdot|s) \; || \; \tilde{\pi}^{-i}(\cdot|s_t)\right)$, we exploit the temporal improvement assumption for discrete time dynamics, parameterized by $\theta^{-i}$:

\begin{equation}
\theta_t^{-i} \; \approx \; \theta_{t-1}^{-i} \; + \; \eta \nabla_{\theta^{-i}} \; \mathbb{E}_{\tilde{\pi}_{t-1}^{i}, \; \tilde{\pi}^{-i}_{\theta}} \left[Q^{-i}(a^i, \; \mathbf{a}^{-i}, \; s) \right],
\end{equation}

\noindent for some $\eta > 0$. However, the opponent model cannot be updated like this since neither $\tilde{\pi}^{-i}_{\theta_t}$ nor $\theta^{-i}_t$ are directly available to agent $i$. Hence, we take our best approximation $\rho_{\psi_t}^i$ which is our opponent model and $\psi_t^i$ and update as

\begin{equation}
\psi_t^i \; \leftarrow \; \psi_{t-1}^i \; + \; \eta \nabla_{\psi^{i}} \; \mathbb{E}_{a^i \sim \tilde{\pi}_{t-1}^{i}, \; \mathbf{a}^{-i} \sim \rho_{\psi_t^{i}}} \left[Q^{-i}(a^i, \; \mathbf{a}^{-i}, \; s) \right].
\label{Opobj}
\end{equation}

We point out that the $Q$ mentioned above can represent any type of critic function, such as Q-function, soft Q-function or advantage function.
\\

To summarize, firstly, we derive a learning objective for agent $i$'s policy $\pi^i$; secondly, we propose to exploit the temporal improvement assumption to guide the opponent model evolution. 

\section{Multi-Agent Actor-Critic with Time Dynamical Opponent model (TDOM-AC)}
\label{ac}
With the proposed TDOM, we introduce Multi-Agent Actor-Critic with Time Dynamical Opponent Model (TDOM-AC). TDOM-AC follows the CTDE framework \cite{foerster2016learning,lowe2017multi}. There are three main modules in the proposed TDOM-AC: Centralized Q-function $Q(s,a^i,\textbf{a}^{-i})$, opponent model $\rho(\cdot|s)$ and policy $\pi(\cdot|s,\hat{\textbf{a}}^{-i})$. We further use neural networks (NNs) as function approximators, particularly applicable in high-dimensional and/or continuous multi-agent tasks. For an individual agent, $i$, the three modules are parameterized by $\phi^i$, $\theta^i$ and $\psi^i$, respectively. The functions are updated using stochastic gradient based optimization with learning rates $\eta_\cdot$:

\begin{equation} \label{equ:Q&policy}
\begin{aligned}
    \centering
    \phi^i_{t+1} \; &\leftarrow \; \ \phi^i_{t} + \eta_\phi \hat{\nabla}_{\phi^i} J(\phi_t^i) \\
    \theta^i_{t+1} \; &\leftarrow \; \ \theta^i_{t} + \eta_\theta \hat{\nabla}_{\theta^i} J(\theta_t^i)
\end{aligned}
\end{equation}

and as elucidated in subsection \ref{TDOM},

\begin{equation} \label{equ:opp_model}
\begin{aligned}
    \centering
    \psi^i_{t+1} \; &\leftarrow \; \ \psi^i_{t} + \eta_\psi \hat{\nabla}_{\psi^i} J(\psi_t^i).
\end{aligned}
\end{equation}

We would like to clarify that although Equations \ref{equ:Q&policy} and \ref{equ:opp_model} perform similar operations, their underlying idea is different. We can interpret Equation \ref{equ:Q&policy} as an approximation of a policy improvement and evaluation step without running it until convergence. However, Equation \ref{equ:opp_model} does not follow this idea. Instead, this update is based on the temporal improvement assumption with the underlying goal of minimizing the Kullback-Leibler divergence to the true marginal opponent policies $\tilde{\pi}^{-i}$ instead of policy improvement.

In the proposed TDOM-AC, experience replay buffer $D$ is used \cite{mnih2015human}, where the off-policy experiences of all agents are recorded. In a scenario with $N$ agents, at time step t, a tuple $[s_t,s_{t+1},a^{1}_t,..., a^{N}_t, r^{1}_t,..., r^{N}_t ]$ is recorded.  

We adopt the maximum entropy reinforcement learning (MERL) framework \cite{haarnoja2018soft} to enable a richer exploration and a better learning stability. It is easy to see that the derivation still holds. We merely omit the adjustments in the previous sections for the purpose of readability. The centralized soft Q-function parameters can be trained to minimize the soft Bellman residual:
\begin{equation}
\begin{aligned}
 J(\phi^i) \; =& \; \mathbb{E}_{(s_t,a_t,\mathbf{a}^{-i}_t,s_{t+1})\sim D}\frac{1}{2}[ Q^i_{\phi}(s_t,a_t,\mathbf{a}^{-i}_t)-(r^i_t+\gamma V(s_{t+1}))]^2,
\end{aligned}
\end{equation}

\noindent where the value function $V$ is implicitly parameterized by the soft Q-function \cite{haarnoja2018soft} parameters. The objective function becomes:
\begin{equation}
\begin{aligned}
    \centering
    J(\phi^i) \; =& \; -\mathbb{E}_{(s_t, a_t, s_{t+1}) \sim D, \; \hat{\mathbf{a}}^{-i}_{t+1} \sim \rho^{i}_\psi, \; \hat{a}^i_{t+1} \sim \pi^i_\theta} \Big[ \Big( Q^i_{\phi}(s_t, a^i_t, \mathbf{a}^{-i}_{t})-\big(r^i(s_t, a^i_t, \mathbf{a}^{-i}_{t}) \\
    & \; + \gamma \big( Q^i_{\overline{\phi}}(s_{t+1},\hat{a}^i_{t+1},\hat{\mathbf{a}}^{-i}_{t+1}) -\alpha\log\pi^i(\hat{a}^i_{t+1}|s_{t+1},\hat{\mathbf{a}}^{-i}_{t+1}) \\&\; -\alpha\log\rho_\psi^{i}(\hat{\mathbf{a}}^{-i}_{t+1}|s_{t+1})\big)\big) \Big)^2\Big].
\end{aligned}
\label{qfinal}
\end{equation}

The $Q^i_{\overline{\phi}}$ is the target soft Q-network that has the same structure as $Q^i$ and is parameterized by $\overline{\phi^i}$, but updated through exponentially moving average of the soft Q-function weights \cite{mnih2015human}. 

According to the MERL objective, the TDOM-based policy is learned by directly minimizing the expected KL-divergence between normalized centralized soft Q-function: 
\begin{equation}
\begin{aligned}
    \centering
    J(\theta^i) \; = \; & \mathbb{E}_{s \sim D, \; \hat{\mathbf{a}}^{-i}_{t+1} \sim \rho^{i}_\psi}[Q_\phi^i(s,a^i,\hat{\mathbf{a}}^{-i}) -\alpha\log\pi^i_{\theta}(a^i|s,\hat{\mathbf{a}}^{-i})],
\end{aligned}
\label{pu}
\end{equation}
\noindent where $\alpha$ is the temperature parameter that determines the relative importance of the entropy term versus the reward, thus controls the stochasticity of the optimal policy. In order to achieve a low variance estimator of $J(\theta^i)$, we apply the reparameterization trick \cite{kingma2014autoencoding} for modeling the policy:
\begin{equation}
\begin{aligned}
\centering
    a^i_t \; = \; f_{\theta}^i(\epsilon;s,\mathbf{a}^{-i}),
\end{aligned}
\end{equation}
where $\epsilon_t$ is a noise vector that is sampled from a fixed distribution. A common choice is a Gaussian distribution $\mathcal{N}$.
We can now rewrite the objective in Equation \ref{pu} as
\begin{equation}
\begin{aligned}
\centering
    J(\theta^i) \;=\;& \mathbb{E}_{s \sim D, \; \hat{\mathbf{a}}^{-i}_{t+1} \sim \rho^{i}_\psi, \; \epsilon \sim \mathcal{N}}[Q^i(s,f^i_{\theta}(\epsilon;s,\mathbf{a}^{-i}),\mathbf{a}^{-i})\\&-\alpha\log\pi^i_{\theta}(f^i_{\theta}(\epsilon;s,\mathbf{a}^{-i})|s,\mathbf{a}^{-i})].
\end{aligned}
\label{pifinal}
\end{equation}

Let $\mathbf{Q}^{-i}(s,a^i,\hat{\mathbf{a}}^{-i}) \; := \; \sum_{j \in -i} Q_\phi^j(s, a^i, \hat{\mathbf{a}}^{-i})$. Then, according to Equation \ref{Opobj}, the objective for the TDOM model can be written as
\begin{equation}
\begin{aligned}
    \centering
    J(\psi^i) \; =& \; \mathbb{E}_{s \sim D, \; \hat{\mathbf{a}}^{-i} \sim \rho^i_{\psi}, \; a^i\sim \pi^i_\theta} [\mathbf{Q}^{-i}(s,a^i,\hat{\mathbf{a}}^{-i})-\alpha\log \rho^i_{\psi}(\hat{\mathbf{a}}^{-i}|s)].
\end{aligned}
\end{equation}
However, in mixed cooperative-competitive environments, agents may have conflicting interests which can \textit{neutralize} the gradient in this formulation. We illustrate this by an example of a two-player zero-sum Markov game:
\begin{equation}
r^1(s, a^1, a^2)=-r^2(s, a^1, a^2), \qquad \forall s \in \mathcal{S}, \; (a^1, a^2) \in \mathcal{A}^2.
\end{equation}
\begin{assumption}
The Q-function approximations $Q_\phi^1$ and $Q_\phi^2$ for agent 1 and agent 2 respectively, have converged to their true functions $Q^1_{\pi_\theta^1, \pi_\theta^2}$ and $Q^2_{\pi_\theta^1, \pi_\theta^2}$.
\end{assumption}
\begin{theorem}
In this setting, the gradient $\nabla_{\psi^i}J(\psi^i)$ is exclusively determined by entropy terms.
\end{theorem}

\begin{proof}
% See Appendix \ref{proof1}

With $p(\tau)$ denoting the trajectory distribution, observe that the structure of the true $Q^1_{\pi_\theta^1, \pi_\theta^2}$ is:
\begin{equation}
\begin{aligned}
\centering
\quad & Q^1_{\pi_\theta^1, \pi_\theta^2}(s_0, a^1_0, a^2_0) \\\triangleq \quad & r^1(s_0, a^1_0, a^2_0) + \mathbb{E}_{s \sim p(s_1 | a^1_0, a^2_0)}\left(\gamma V^1_{\pi_\theta^1,\pi_\theta^2}(s_1)\right) \\
\triangleq \quad & r^1(s_0, a^1_0, a^2_0)+\mathbb{E}_{\tau \sim p(\tau)}\Bigg[\sum_{t = 1}^\infty \gamma^{t}\Big( r_t^1(s_t, a^1_t, a^2_t) \mathcal{H}(\pi_\theta^1(a^1_t|s_t, a_t^2) \rho_\psi^1(a^2_t|s_t))\Big)\Bigg] \\
= \quad & -r^2(s_0, a^1_0, a^2_0)-\mathbb{E}_{\tau \sim p(\tau)} \Bigg[ \sum_{t = 1}^\infty \gamma^{t} \Big( r_t^2(s_t, a^1_t, a^2_t) + \mathcal{H}(\pi_\theta^1(a^1_t|s_t, a_t^2) \rho_\psi^1(a^2_t|s_t))\Big)\Bigg] \\
= \quad & \bigg( \sum_{t=1}^\infty \gamma^{t}\mathcal{H}(\pi_\theta^2(a^2_t|s_t, a_t^1) \rho_\psi^2(a^1_t|s_t)) -
\gamma^{t} \mathcal{H}(\pi_\theta^1(a^1_t|s_t, a_t^2) \rho_\psi^1(a^2_t|s_t)) \bigg) \\&- Q^2_{\pi_\theta^1, \pi_\theta^2}(s_0, a^1_0, a^2_0),
\end{aligned}
\end{equation}
\noindent where $\mathcal{H}(\cdot)$ denotes Shannon entropy.
For lighter notation, let 
\begin{equation}
\begin{aligned}
\mathcal{E}= \bigg( \sum_{t=1}^\infty \gamma^{t}\mathcal{H}(\pi^2(a^2_t|s_t, a_t^1) \rho_\psi^2(a^1_t|s_t)) -\gamma^{t}\mathcal{H}(\pi^1(a^1_t|s_t, a_t^2) \rho_\psi^1(a^2_t|s_t)) \bigg).
\end{aligned}
\end{equation}

Then, we can determine the gradient as:
\begin{equation}
\begin{aligned}
\centering
\nabla_{\psi^i}J(\psi^i) \; &= \; \mathbb{E}_{\tau \sim p, \; \epsilon \sim \mathcal{N}}\big[\nabla_{\psi^i}Q^2_{\pi^1, \pi^2}(s_0, f^i_\psi(\epsilon; s_0)) - \nabla_{\psi^i}Q^2_{\pi^1, \pi^2}(s_0, f^i_\psi(\epsilon; s_0))\\ 
& \quad \; + \nabla_{\psi^i} \mathcal{E}-\alpha \nabla_{\psi^i}\text{log}(\rho_{\psi}^i (f^i_\psi(\epsilon; s_0)|s_0) \big)\big] \\
& = \; \mathbb{E}_{\tau \sim p, \; \epsilon \sim \mathcal{N}} \big[\nabla_{\psi^i} \mathcal{E} -\alpha \nabla_{\psi^i}\text{log}(\rho_{\psi}^i (f^i_\psi(\epsilon; s_0)|s_0) \big)\big] \\
& = \; \mathbb{E}_{\tau \sim p}\big[\nabla_{\psi^i} \mathcal{E}+\alpha \nabla_{\psi^i}\mathcal{H}(\rho_{\psi}^i (\cdot, \cdot|s_0) \big)\big].
\end{aligned}
\end{equation}
\end{proof} 

To alleviate the potential issue of neutralized gradients, we propose to modify the TDOM objective to be based on empirical data. Specifically, we modify the objective function as
\begin{equation}
\begin{aligned}
    \centering
    J(\psi^i) \; = \; & \mathbb{E}_{(s,\mathbf{a}^{-j}) \sim D, \; \hat{a}^{j} \sim \rho^i_{\psi}} \Big[ \sum_{j \in -i} Q_\phi^j\left(s, \hat{a}^j, \mathbf{a}^{-i \backslash \{j\}}, a^i\right)  -\alpha\log \rho^j_{\psi}(\hat{\mathbf{a}}^{-j}|s)\Big].
\end{aligned}
\label{opfinal}
\end{equation}
% Note that again we use the reparameterization trick $(\hat{a}^1, \hat{a}^2) =: f^i_\psi(\epsilon; s)$ to  allowed to exchange expectation and gradient as shown in \cite{haarnoja2018soft}, while still sampling from the opponent model $\rho_{\psi}^i$.
Note that again we use the reparameterization trick \cite{haarnoja2018soft} in order to be able to exchange expectation and gradient, while still sampling from the opponent model $\rho_{\psi}^i$. The pseudo-code can be found below \ref{algo:TDOM-AC}.

\begin{algorithm}[htbp]
   \caption{Multi-agent Actor-Critic with Time Dynamical Opponent Model (TDOM-AC)}
   \label{algo:TDOM-AC}
\begin{algorithmic}
   \STATE Initialize replay buffer $D$ and randomly initialize $N$ soft Q networks $Q^{1..n}_{\phi_{i..n}}$, $N$ policy networks $\pi^{1..n}_{\theta_{1..n}}$, and  opponent model $\rho^{1..n}_{\psi_{1..n}}$ with parameters $\phi_{i..n}$, $\theta_{1..n}$ and  $\psi_{1..n}$.
    \STATE Initialize the parameters of target networks with $Q^{1..n}_{\overline{\phi}_{1..n}}$
   \FOR{each iteration}
   \STATE Sample $s_0$ according to $p_0(\cdot)$
   \WHILE{Not done}
   \FOR{each agent}
  \STATE Sample $\hat{\textbf{a}}^{-i}_t$ from $\rho^i(\cdot|s_t)$ and $a^{i}_t$ from $\pi^i(\cdot|s_t,\hat{\textbf{a}}^{-i}_t)$ 
   \STATE Combine the true actions $\textbf{a}_t=[a^1_t,...,a^n_t]$ and take one step forward
   \ENDFOR
   \STATE Observe $s_{t+1}$, $\textbf{r}_t=[r^1_t, ..., r^n_t]$ and store $(s_t,\textbf{a}_t, \textbf{r}_t,s_{t+1})$ in $D$
   \STATE Sample minibatches of $N$ transitions from $D$
   \FOR{each agent}
   \STATE Estimate policy gradient according to Equations \ref{qfinal},\ref{pifinal}, and \ref{opfinal}:
   \begin{equation}
    \begin{aligned}
        \phi^i_{t+1} \; &\leftarrow \; \ \phi^i_{t} + \eta_\phi \hat{\nabla}_{\phi^i} J(\phi_t^i) \\
    \theta^i_{t+1} \; &\leftarrow \; \ \theta^i_{t} + \eta_\theta \hat{\nabla}_{\theta^i} J(\theta_t^i)\\
     \psi^i_{t+1} \; &\leftarrow \; \ \psi^i_{t} + \eta_\psi \hat{\nabla}_{\psi^i} J(\psi_t^i).
 \notag
    \end{aligned}
    \end{equation}
   \STATE Update the parameters of target networks $Q^{1..n}_{\overline{\phi}_{1..n}}$
   \ENDFOR
   \ENDWHILE
   \ENDFOR
\end{algorithmic}
\end{algorithm}

\section{Experiments}
\label{sec:exp}
We compare the proposed TDOM-AC to two state-of-the-art algorithms based on opponent modelling: PR2 \cite{wen2019probabilistic} and ROMMEO \cite{tian2019regularized}, which have shown a better performance with respect to the considered measures compared to \textit{Multi-Agent Soft Q-Learning} MASQL \cite{wei2018multiagent} and MADDPG \cite{lowe2017multi} in previous studies. We evaluate the performance of the proposed TDOM-AC methods on a differential game \cite{wei2018multiagent,wen2019probabilistic,tian2019regularized} and the multi-agent particle environments \cite{lowe2017multi}. Those tasks contain fully cooperative and mixed cooperative-competitive objectives with challenging non-trivial equilibria \cite{wen2019probabilistic} and continuous action space. All the experiments are adopted from PR2 and ROMMEO for adequate comparison.

To reduce the performance difference caused solely by entropy regularization, we add an entropy term to the PR2 objective and equip it with a stochastic policy since both TDOM-AC and ROMMEO employ the maximum entropy reinforcement learning framework. This has been shown to yield better exploration and sample efficiency \cite{haarnoja2018soft}.

For the experiment settings, all policies and opponent models use a fully connected multi-layer perceptron (MLP) with two hidden layers of 256 units each, outputting the mean $\mu$ and standard deviation $\sigma$ of a univariate Gaussian distribution. All hidden layers use the leaky-RelU activation function and we adopt the same invertible squashing function technique as \cite{haarnoja2018soft} for the output layer. For the Q-network, we use a fully-connected MLP with two hidden layers of 256 units with leaky-Relu activation function, outputting the Q-value. We employ the Adam optimizer with the learning rate $3e-4$ and batch size 256. The target smoothing coefficient $\tau$, entropy control parameter $\alpha$ and the discount factor $\gamma$ are $0.01,1$, and $0.95$ respectively. All training hyper-parameters are derived from the SAC algorithm (as published in \cite{haarnoja2018soft}) without any additional adaptations. 
\subsection{Differential Game}
The differential Max-of-Two Quadratic Game is a single step continuous action space decision making task, where the gradient update tends to direct the training agent to a sub-optimal point \cite{tian2019regularized}. The reward surface is displayed in the Fig \ref{DIFF-1}. There exists a local maximum $0$ at $(-5, -5)$ and a global maximum $10$ at $(5, 5)$, with a deep valley positioned in the middle. The agents are rewarded by their joint actions, following the rule: $r_1=r_2=max(f_1,f_2)$, where $f_1=0.8*[-(\frac{a_1+5}{3})^2-(\frac{a_2+5}{3})^2]$ and $f_1=[-(\frac{a_1-5}{1})^2-(\frac{a_2-5}{1})^2]+10$. Both of the agents have the same continuous action space in the range $[-10, 10]$.
\begin{figure}[h]
\centering
\includegraphics[width=0.9\columnwidth]{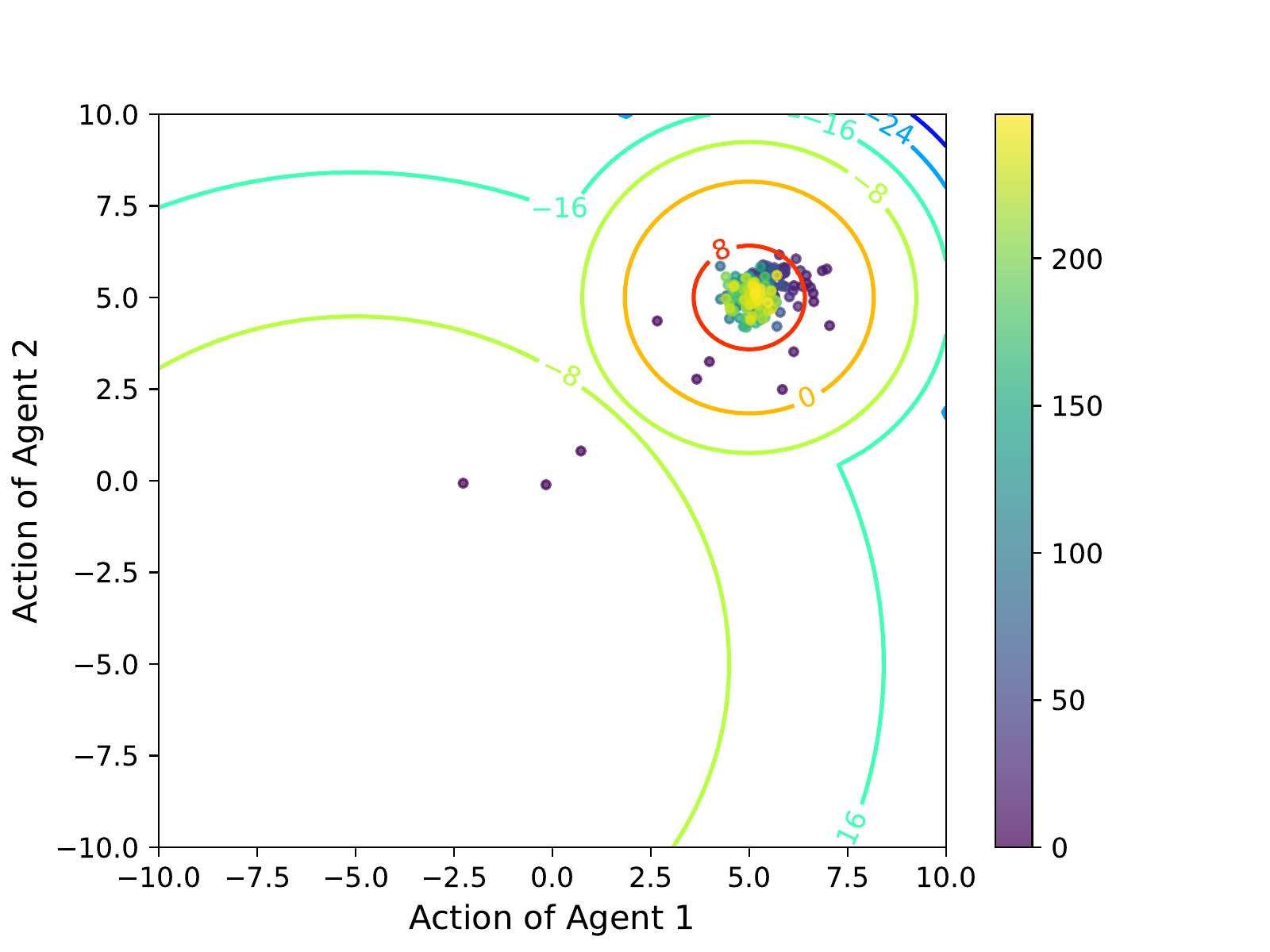}
\caption{Reward surface and learning path of agents trained by TDOM-AC. Scattered points are actions taken at each step, the lighter points are sampled later during training.}
\label{DIFF-1}
\end{figure}
Compared to other state-of-the art approaches, TDOM-AC shows a superior performance. In Fig. \ref{DIFF-2}, the learning path of the proposed TDOM-AC is displayed, where the lighter (yellow) dots are sampled later. This indicates a stable and fast convergence. In Fig. \ref{DIFF-2}, the learning curves of all considered algorithms are displayed. Both TDOM-AC and ROMMEO show a fast and stable convergence. However, ROMMEO fails for some random seeds, resulting in a lower average performance. We note that the maximum-entropy version of PR2 indeed converges faster than the original version \cite{wen2019probabilistic}. Nevertheless, the learning process fluctuates significantly and it suffers from substantial computational cost, see Table \ref{tab:time}.
\begin{figure}[h]
\centering
\includegraphics[width=0.9\columnwidth]{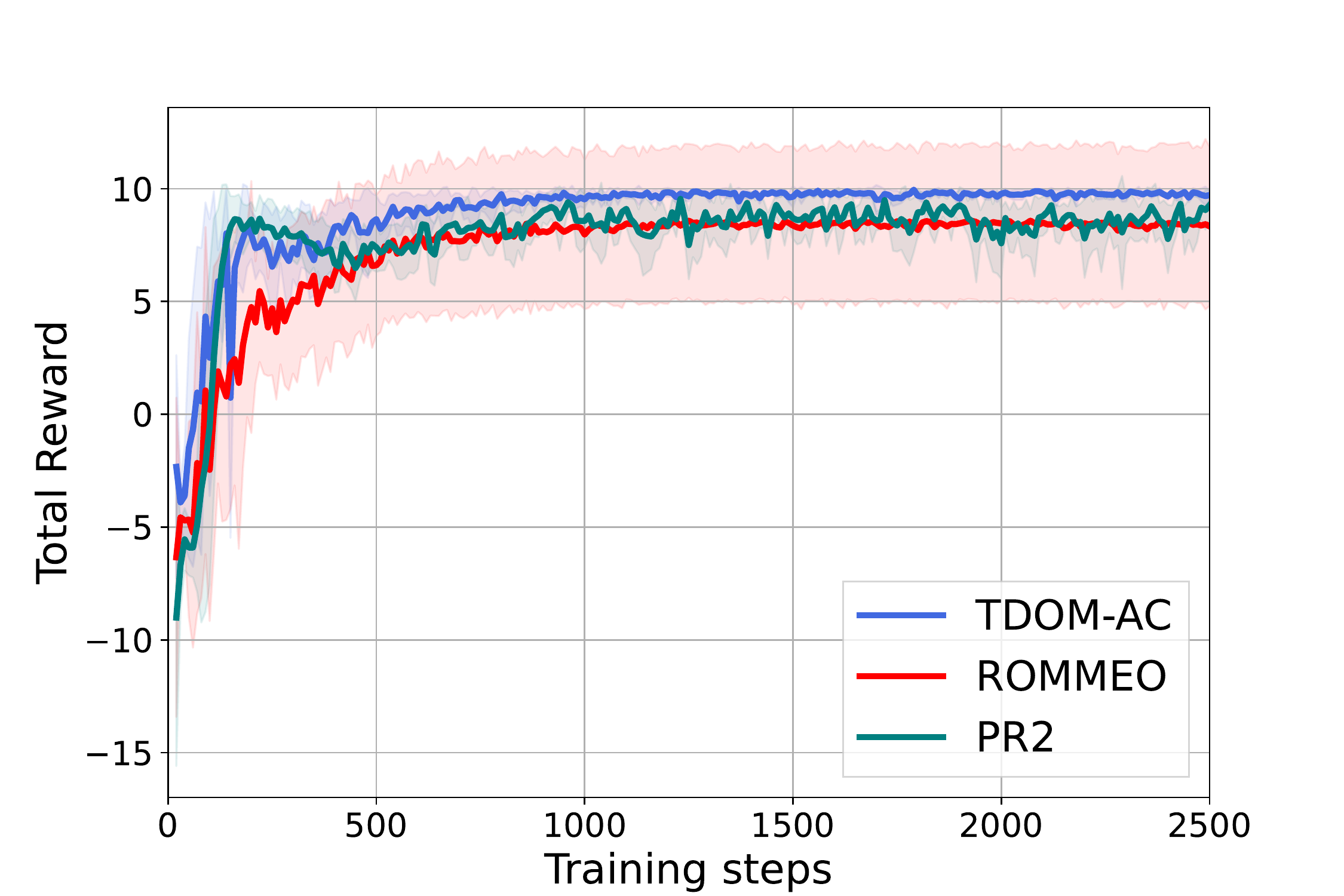}
\caption{Average performance of TDOM-AC and other baselines, where the shaded areas show the 1-SD confidence intervals over multiple random seeds}
\label{DIFF-2}
\end{figure}
\begin{table}
\centering
\begin{tabular}{llll}
\hline
Methods  &TDOM-AC&ROMMEO&PR2 \\
\hline
Running time&$\textbf{0.068s}$&$0.089s$&$0.436s$\\
\hline
\end{tabular}
\caption{Average running time (seconds) per update of different methods}
\label{tab:time}
\end{table}

\subsection{Cooperative Navigation}
\begin{figure}[htbp]
\centering
\includegraphics[width=0.9\columnwidth]{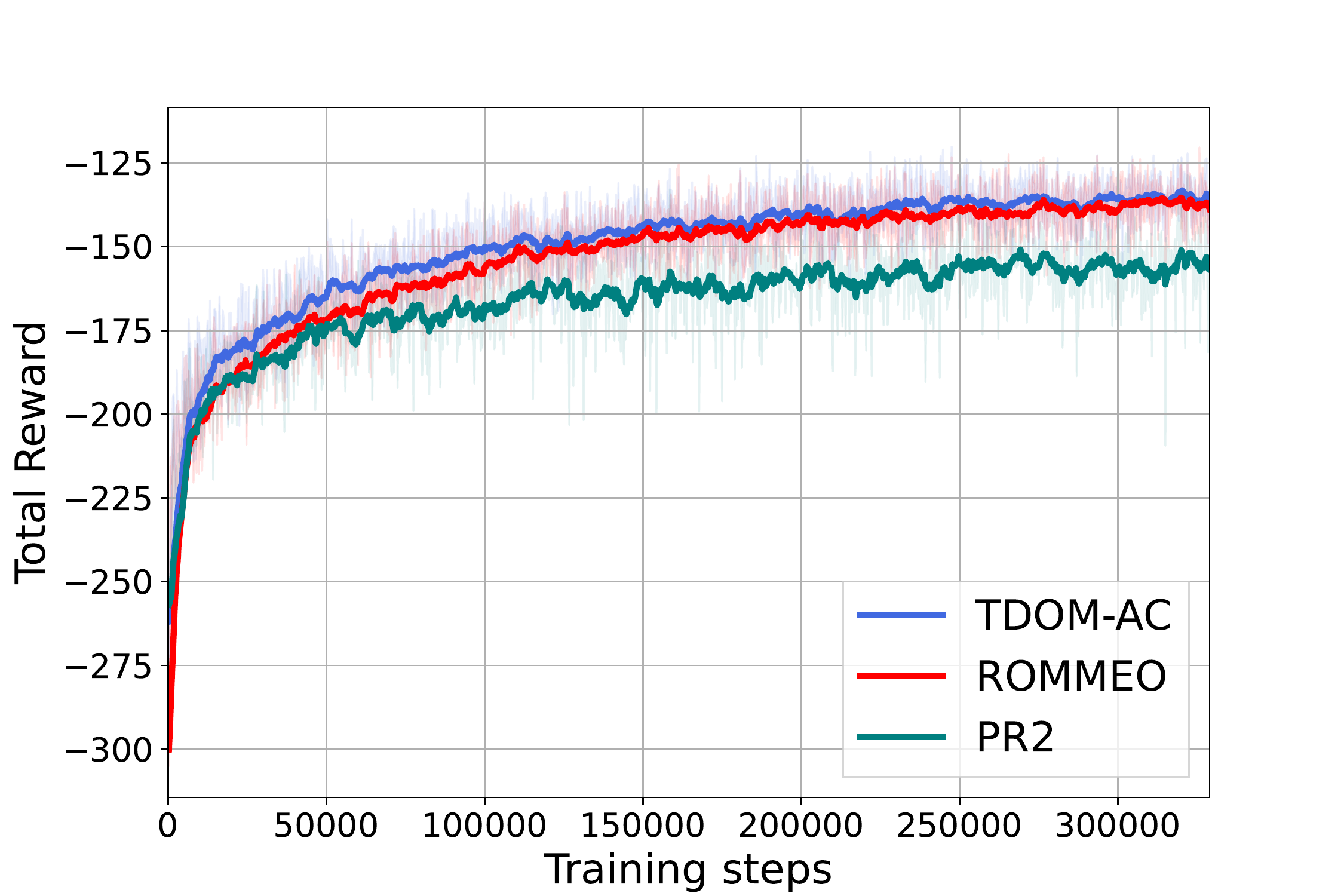}
\caption{Moving average of total reward of TDOM-AC and other baselines on Cooperative Navigation.}
\label{Navi-1}
\end{figure}

Cooperative Navigation is a three-agent fully cooperative task. The three agents should learn to cooperate to reach and cover three randomly generated landmarks. The agents can observe the relative positions of other agents and landmarks and are collectively rewarded based on the proximity of any agent to each landmark. Besides this, the agents are being penalized when colliding with each other. The expected behavior is to "cover" the three landmarks as fast as possible without any collision. The result shows that TDOM-AC outperforms all other considered baseline algorithms in terms of both faster convergence and a better performance, see Fig \ref{Navi-1}. Also, the TDOM-AC attains more accurate opponent behavior prediction, despite the fact that the agents do not have direct access to any opponent action distribution, see Fig \ref{Navi-2}. This is in contrast to ROMMEO, which utilizes a regularized opponent model, the regularization being the KL divergence between the opponent model and the empirical opponent distribution.

\begin{figure}[htbp]
\centering
\includegraphics[width=0.9\columnwidth]{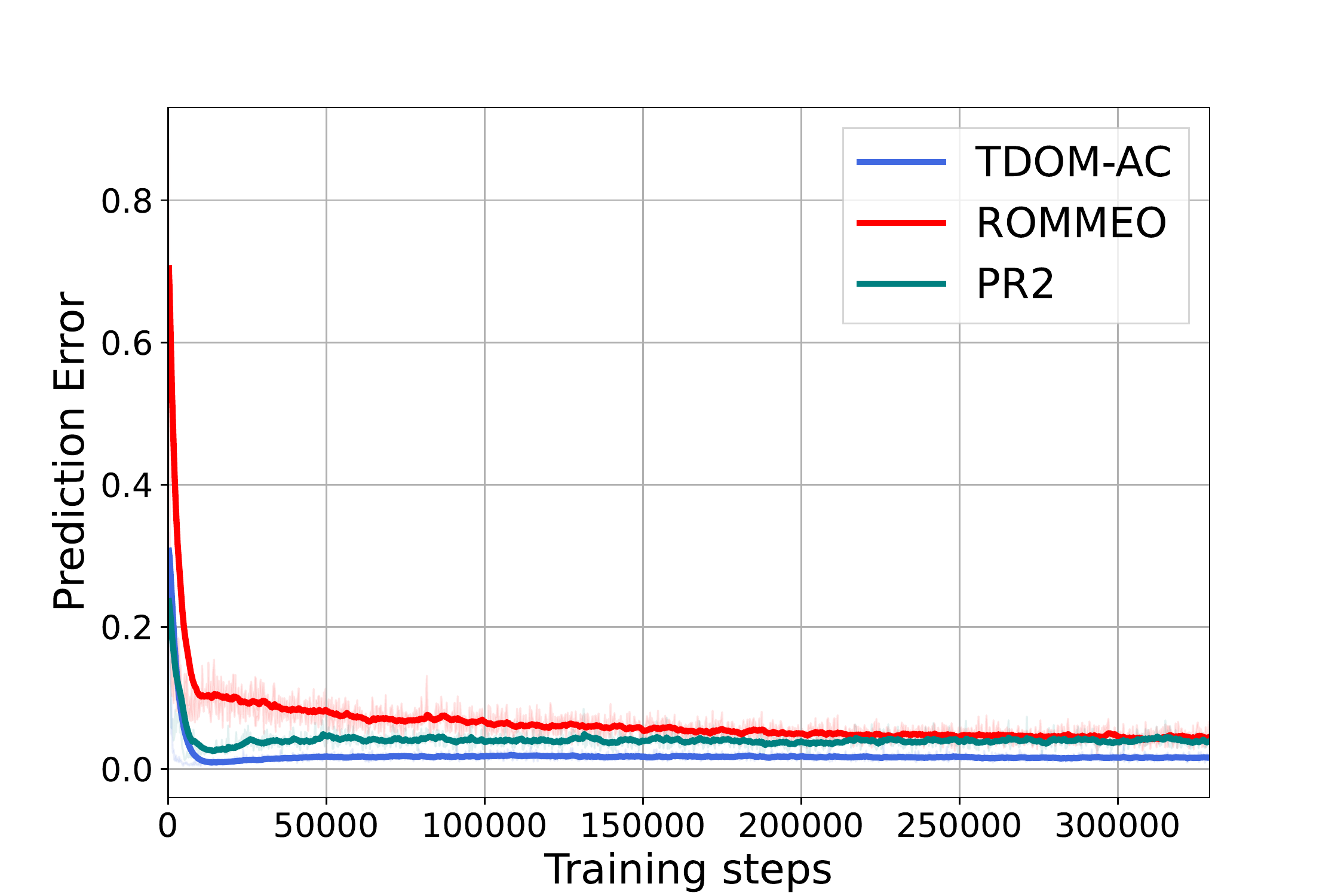}
\caption{The test time opponents' behaviors prediction error of TDOM-AC and other baselines}
\label{Navi-2}
\end{figure}
\ 
\subsection{Predator and Prey}
Predator and Prey is a challenging four-agent mixed cooperative-competitive task. There are three slower cooperating adversaries that try to chase the faster agent in a randomly generated environment with two large landmarks impeding the way. The cooperative adversaries are rewarded for every collision with the agent, while the agent is being penalized for any such collision. All agents can observe the relative positions and velocities of other agents and the positions of the landmarks.

For this task, we train all the algorithms for 0.6M steps and compare the normalized average episode advantage score (the sum of agent's rewards in an episode - the sum of adversaries' rewards in an episode \cite{wen2019probabilistic,lowe2017multi}. We evaluate the performance of the different algorithms by letting the cooperative adversaries trained by one algorithm play against an agent trained by another algorithm and vice versa. A higher score means the agent (prey) performs better than the cooperative adversaries (predators), while a lower score means that the cooperative adversaries have a superior policy over the agent. Table \ref{tab:plain} shows that the TDOM-AC performs best on both prey $(0.999)$ and predator $(0.547)$ side.
\begin{table}[h]
\centering
\begin{tabular}{llll|l}
\hline
\textcolor{red}{Ag} vs. \textcolor{blue}{Ads}                &  \textcolor{blue}{TDOM-AC} &  \textcolor{blue}{ROMMEO} &  \textcolor{blue}{PR2}&Mean \\ \hline
\textcolor{red}{TDOM-AC}   &0.967&1.000&0.999& \textbf{0.989} \\
\textcolor{red}{ROMMEO}   &0.674&0.997&0.981&0.884 \\
\textcolor{red}{PR2}   &0.000&0.722&0.313&0.345\\\hline
Mean  &\textbf{0.547}&0.906&0.764&N/A\\
\hline
\end{tabular}
\caption{Comparison of different model settings (Agent vs. Adversaries). The values are the normalized average episode advantage scores.}
\label{tab:plain}
\end{table}

\section{Conclusion}
\label{sec:conclusion}
In this work, we propose a novel time dynamical opponent model called TDOM. It supports mixed cooperative-competitive tasks with a low computational cost. Furthermore, we introduce the  TDOM-AC algorithm and demonstrate its' superior performance compared to other state-of-the-art methods on multiple challenging benchmarks. In the future, we plan to omit the centralized training and instead also model opponent Q-function parameters as time dynamical latent variables, thereby relying exclusively on past opponent actions for training. Also, we would like to evaluate the proposed approach on partially observable environments where the agent does not share its observation space with all opponents.
\bibliography{main}

\begin{thebibliography}{10}
\expandafter\ifx\csname url\endcsname\relax
  \def\url#1{\texttt{#1}}\fi
\expandafter\ifx\csname urlprefix\endcsname\relax\def\urlprefix{URL }\fi
\expandafter\ifx\csname href\endcsname\relax
  \def\href#1#2{#2} \def\path#1{#1}\fi

\bibitem{du2021learning}
Y.~Du, B.~Liu, V.~Moens, Z.~Liu, Z.~Ren, J.~Wang, X.~Chen, H.~Zhang, Learning
  correlated communication topology in multi-agent reinforcement learning, in:
  Proceedings of the 20th International Conference on Autonomous Agents and
  MultiAgent Systems, 2021, pp. 456--464.

\bibitem{vinyals2019grandmaster}
O.~Vinyals, I.~Babuschkin, W.~M. Czarnecki, M.~Mathieu, A.~Dudzik, J.~Chung,
  D.~H. Choi, R.~Powell, T.~Ewalds, P.~Georgiev, et~al., Grandmaster level in
  starcraft ii using multi-agent reinforcement learning, Nature 575~(7782)
  (2019) 350--354.

\bibitem{brown2019superhuman}
N.~Brown, T.~Sandholm, Superhuman ai for multiplayer poker, Science 365~(6456)
  (2019) 885--890.

\bibitem{OpenAI_dota}
OpenAI, Openai five, \url{https://blog.openai.com/openai-five/} (2018).

\bibitem{yin2022point}
Y.~Yin, X.~Bu, P.~Zhu, W.~Qian, Point-to-point consensus tracking control for
  unknown nonlinear multi-agent systems using data-driven iterative learning,
  Neurocomputing.

\bibitem{yuan2022suboptimal}
S.~Yuan, C.~Yu, P.~Wang, Suboptimal linear quadratic tracking control for
  multi-agent systems, Neurocomputing.

\bibitem{huttenrauch2019deep}
M.~H{\"u}ttenrauch, S.~Adrian, G.~Neumann, et~al., Deep reinforcement learning
  for swarm systems, Journal of Machine Learning Research 20~(54) (2019) 1--31.

\bibitem{hernandez2017survey}
P.~Hernandez-Leal, M.~Kaisers, T.~Baarslag, E.~M. de~Cote, A survey of learning
  in multiagent environments: Dealing with non-stationarity, arXiv preprint
  arXiv:1707.09183.

\bibitem{sutton1992reinforcement}
R.~S. Sutton, A.~G. Barto, R.~J. Williams, Reinforcement learning is direct
  adaptive optimal control, IEEE Control Systems Magazine 12~(2) (1992) 19--22.

\bibitem{lowe2017multi}
R.~Lowe, Y.~Wu, A.~Tamar, J.~Harb, P.~Abbeel, I.~Mordatch, Multi-agent
  actor-critic for mixed cooperative-competitive environments, arXiv preprint
  arXiv:1706.02275.

\bibitem{brown1951iterative}
G.~W. Brown, Iterative solution of games by fictitious play, Activity analysis
  of production and allocation 13~(1) (1951) 374--376.

\bibitem{tian2019regularized}
Z.~Tian, Y.~Wen, Z.~Gong, F.~Punakkath, S.~Zou, J.~Wang, A regularized opponent
  model with maximum entropy objective, arXiv preprint arXiv:1905.08087.

\bibitem{strogatz2018nonlinear}
S.~H. Strogatz, Nonlinear dynamics and chaos with student solutions manual:
  With applications to physics, biology, chemistry, and engineering, CRC press,
  2018.

\bibitem{foerster2016learning}
J.~N. Foerster, Y.~M. Assael, N.~De~Freitas, S.~Whiteson, Learning to
  communicate with deep multi-agent reinforcement learning, arXiv preprint
  arXiv:1605.06676.

\bibitem{wen2019probabilistic}
Y.~Wen, Y.~Yang, R.~Luo, J.~Wang, W.~Pan, Probabilistic recursive reasoning for
  multi-agent reinforcement learning, arXiv preprint arXiv:1901.09207.

\bibitem{kamdar2018state}
R.~Kamdar, P.~Paliwal, Y.~Kumar, A state of art review on various aspects of
  multi-agent system, Journal of Circuits, Systems and Computers 27~(11) (2018)
  1830006.

\bibitem{zhang2021multi}
K.~Zhang, Z.~Yang, T.~Ba{\c{s}}ar, Multi-agent reinforcement learning: A
  selective overview of theories and algorithms, Handbook of Reinforcement
  Learning and Control (2021) 321--384.

\bibitem{foerster2018counterfactual}
J.~Foerster, G.~Farquhar, T.~Afouras, N.~Nardelli, S.~Whiteson, Counterfactual
  multi-agent policy gradients, in: Proceedings of the AAAI Conference on
  Artificial Intelligence, Vol.~32, 2018.

\bibitem{yang2018mean}
Y.~Yang, R.~Luo, M.~Li, M.~Zhou, W.~Zhang, J.~Wang, Mean field multi-agent
  reinforcement learning, in: International Conference on Machine Learning,
  PMLR, 2018, pp. 5571--5580.

\bibitem{yang2020multi}
Y.~Yang, Y.~Wen, J.~Wang, L.~Chen, K.~Shao, D.~Mguni, W.~Zhang, Multi-agent
  determinantal q-learning, in: International Conference on Machine Learning,
  PMLR, 2020, pp. 10757--10766.

\bibitem{rashid2018qmix}
T.~Rashid, M.~Samvelyan, C.~Schroeder, G.~Farquhar, J.~Foerster, S.~Whiteson,
  Qmix: Monotonic value function factorisation for deep multi-agent
  reinforcement learning, in: International Conference on Machine Learning,
  PMLR, 2018, pp. 4295--4304.

\bibitem{zhang2021fop}
T.~Zhang, Y.~Li, C.~Wang, G.~Xie, Z.~Lu, Fop: Factorizing optimal joint policy
  of maximum-entropy multi-agent reinforcement learning, in: International
  Conference on Machine Learning, PMLR, 2021, pp. 12491--12500.

\bibitem{son2019qtran}
K.~Son, D.~Kim, W.~J. Kang, D.~E. Hostallero, Y.~Yi, Qtran: Learning to
  factorize with transformation for cooperative multi-agent reinforcement
  learning, in: International Conference on Machine Learning, PMLR, 2019, pp.
  5887--5896.

\bibitem{sunehag2017value}
P.~Sunehag, G.~Lever, A.~Gruslys, W.~M. Czarnecki, V.~Zambaldi, M.~Jaderberg,
  M.~Lanctot, N.~Sonnerat, J.~Z. Leibo, K.~Tuyls, et~al., Value-decomposition
  networks for cooperative multi-agent learning, arXiv preprint
  arXiv:1706.05296.

\bibitem{tian2020learning}
Z.~Tian, S.~Zou, I.~Davies, T.~Warr, L.~Wu, H.~B. Ammar, J.~Wang, Learning to
  communicate implicitly by actions, in: Proceedings of the AAAI Conference on
  Artificial Intelligence, Vol.~34, 2020, pp. 7261--7268.

\bibitem{zheng2018deep}
Y.~Zheng, Z.~Meng, J.~Hao, Z.~Zhang, T.~Yang, C.~Fan, A deep bayesian policy
  reuse approach against non-stationary agents, in: Proceedings of the 32nd
  International Conference on Neural Information Processing Systems, 2018, pp.
  962--972.

\bibitem{albrecht2018autonomous}
S.~V. Albrecht, P.~Stone, Autonomous agents modelling other agents: A
  comprehensive survey and open problems, Artificial Intelligence 258 (2018)
  66--95.

\bibitem{jordan1999introduction}
M.~I. Jordan, Z.~Ghahramani, T.~S. Jaakkola, L.~K. Saul, An introduction to
  variational methods for graphical models, Machine learning 37~(2) (1999)
  183--233.

\bibitem{foerster2017learning}
J.~N. Foerster, R.~Y. Chen, M.~Al-Shedivat, S.~Whiteson, P.~Abbeel,
  I.~Mordatch, Learning with opponent-learning awareness, arXiv preprint
  arXiv:1709.04326.

\bibitem{littman1994markov}
M.~L. Littman, Markov games as a framework for multi-agent reinforcement
  learning, in: Machine learning proceedings 1994, Elsevier, 1994, pp.
  157--163.

\bibitem{shapley1953stochastic}
L.~S. Shapley, Stochastic games, Proceedings of the national academy of
  sciences 39~(10) (1953) 1095--1100.

\bibitem{mnih2015human}
V.~Mnih, K.~Kavukcuoglu, D.~Silver, A.~A. Rusu, J.~Veness, M.~G. Bellemare,
  A.~Graves, M.~Riedmiller, A.~K. Fidjeland, G.~Ostrovski, et~al., Human-level
  control through deep reinforcement learning, nature 518~(7540) (2015)
  529--533.

\bibitem{haarnoja2018soft}
T.~Haarnoja, A.~Zhou, P.~Abbeel, S.~Levine, Soft actor-critic: Off-policy
  maximum entropy deep reinforcement learning with a stochastic actor, in:
  International conference on machine learning, PMLR, 2018, pp. 1861--1870.

\bibitem{kingma2014autoencoding}
D.~P. Kingma, M.~Welling, Auto-encoding variational bayes (2014).
\newblock \href {http://arxiv.org/abs/1312.6114} {\path{arXiv:1312.6114}}.

\bibitem{wei2018multiagent}
E.~Wei, D.~Wicke, D.~Freelan, S.~Luke, Multiagent soft q-learning, in: 2018
  AAAI Spring Symposium Series, 2018.

\end{thebibliography}
\end{document}